%% file: arxiv_main.tex
\documentclass[twoside]{article}

\usepackage[accepted]{aistats2024Arxiv}
\usepackage{subfigure}
\usepackage{placeins}
\usepackage{algorithm}
\usepackage{algcompatible}
\usepackage{hyperref}
\usepackage{cleveref}
\usepackage{physics}
\usepackage{mathtools}
\usepackage{nicefrac}
\usepackage{bbm}
\usepackage{amsfonts}
\usepackage{amsthm}
\usepackage{amsmath, amssymb}

\usepackage[suppress]{color-edits}
\addauthor{sz}{red}
\addauthor{lb}{blue}

\newcommand{\set}[1]{\{ #1 \}}

\input{notations}

\input{notations2}
\allowdisplaybreaks

\begin{document}

%

%

\twocolumn[

\aistatstitle{Near Optimal Adversarial Attack on Stochastic Bandits and Defenses with Smoothed Responses}

\aistatsauthor{ Shiliang Zuo }

\aistatsaddress{ University of Illinois Urbana-Champaign } ]

\begin{abstract}
\input{sections/abstract}
\end{abstract}

\input{sections/intro}

\input{sections/prelim}

\input{sections/UCBattack}

\input{sections/thompson}

\input{sections/lowerBound}

\input{sections/smoothed}

\input{sections/experiments}

\input{sections/conclusion}

\bibliography{allrefs}

\onecolumn

\aistatstitle{Supplementary Materials}

\input{appendix/proofs_UCB}
\input{appendix/proofs_TS}

\input{appendix/proofs_lowerBound}
\input{appendix/proofs_smooth}



\end{document}

%% file: notations.tex
\newcommand{\Ex}{\mathbb{E}}
\newcommand{\Var}{\mathbb{V}}
\newcommand{\eps}{\varepsilon}

\newcommand{\IGNORE}[1]{}
\newcommand{\hatO}{\widehat{O}}
\newcommand{\cN}{\mathcal{N}}

\newcommand{\hatmu}{\hat{\mu}}

\newtheorem{theorem}{Theorem}
\newtheorem{lemma}{Lemma}
\newtheorem{remark}{Remark}

\DeclarePairedDelimiter{\ceil}{\lceil}{\rceil}

%% file: notations2.tex
\newcommand{\REW}{\mathtt{REW}}
\newcommand{\OPT}{\mathtt{OPT}}

\newtheorem{fact}{Fact}
\newtheorem{prop}{Proposition}
\newcommand{\ALG}{\mathtt{ALG}}
\newcommand{\Reg}{\mathtt{REG}}

%% file: sections/abstract.tex
\szdelete{ I study a stochastic multi-arm bandit problem where rewards are subject to adversarial corruption. I propose a novel attack strategy that manipulates a learner employing the upper-confidence-bound (UCB) algorithm into pulling some non-optimal target arm $T - o(T)$ times with a cumulative cost that scales as $\widehat{O}(\sqrt{\log T})$, where $T$ is the number of rounds. I also prove the first lower bound on the cumulative attack cost. The lower bound matches the upper bound up to $O(\log \log T)$ factors, showing the proposed attack strategy to be near optimal. }

I study adversarial attacks against stochastic bandit algorithms. At each round, the learner chooses an arm, and a stochastic reward is generated. The adversary strategically adds corruption to the reward, and the learner is only able to observe the corrupted reward at each round. Two sets of results are presented in this paper. The first set studies the optimal attack strategies for the adversary. The adversary has a target arm he wishes to promote, and his goal is to manipulate the learner into choosing this target arm $T - o(T)$ times. I design attack strategies against UCB and Thompson Sampling that only spends $\widehat{O}(\sqrt{\log T})$ cost. Matching lower bounds are presented, and the vulnerability of UCB, Thompson sampling and $\eps$-greedy are exactly characterized. The second set studies how the learner can defend against the adversary. Inspired by literature on smoothed analysis and behavioral economics, I present two simple algorithms that achieve a competitive ratio arbitrarily close to 1. 

%% file: sections/intro.tex
\section{INTRODUCTION}
Stochastic multi-arm bandit is a framework for sequential decision-making with partial feedback. In its most basic form, a learner interacts with a set of arms giving stochastic rewards, and in each timestep, the learner is able to observe and collect the realized reward of one chosen arm. This framework models many real-world applications, including news recommendation \cite{li2010contextual}, advertisements displayment \cite{chapelle2014simple}, and medical experiments \cite{kuleshov2014algorithms}. Past works have extensively studied algorithms for optimizing the regret in the multi-arm bandit problem (for an overview see~\cite{bubeck2012regret}), and some well-known algorithms include the upper-confidence-bound (UCB) algorithm~\cite{auer2002finite}, the Thompson sampling algorithm~\cite{agrawal2012analysis}, and the $\eps$-greedy algorithm~\cite{auer2002finite}. When deploying multi-arm bandit algorithms in practice, it is crucial to understand the robustness of these algorithms. 

\szdelete{ However, most previous works did not address the issue of whether these algorithms are trustworthy in practice. \lbcomment{Add a few citations here to back up your claim.} In particular, how do these algorithms respond when faced with adversarial attacks? }
Recently, Jun et al.~\cite{jun2018adversarial} initiated studying adversarial attacks on multi-arm bandit algorithms, taking a first step towards understanding the reliability and robustness of these algorithms. In the adversarial attack scenario, an adversary sits between the learner and the environment. The adversary can add corruption to the reward, and he has a specific target arm he wishes to promote. The adversary's goal is to manipulate the learner into choosing this target arm almost always by only strategically adding small corruptions to the rewards. \cite{jun2018adversarial} showed that when the learner is employing the UCB or $\eps$-greedy algorithm, the adversary can hijack the algorithm into choosing the target arm by only spending $O(\log T)$ corruption budget, and they left open the question on whether $O(\log T)$ is indeed the optimal attack cost. This work fully characterizes the optimal attack strategy against UCB, Thompson sampling, and $\eps$-greedy. I show optimal attack strategies with $\hat{O}(\sqrt{\log T})$ attack cost against UCB and Thompson sampling, and provide matching lower bounds. I also give a $\Omega(\log T)$ lower bound on the attack cost against $\eps$-greedy. Hence, the vulnerability of UCB, Thompson sampling, and $\eps$-greedy are exactly characterized. 





\lbdelete{You might want to add a transition sentence here to emphasize the novelty of your work before you describe the setting. Compared with previous works, does your work address trustworthiness issues? Or does your work even advance this problem to the next level.}
\szdelete{ In this work, I study adversarial attacks against the UCB algorithm. I design a novel attack strategy that improves the attack cost in~\cite{jun2018adversarial}. In the stochastic multi-arm bandit setting, a learner interacts with the environment and learns from partial feedback over a time horizon of $T$ rounds. The environment consists of $K$ arms. At each round $t$, the learner chooses an arm $a_t$ and observes a random reward $r_t^0$ generated by this arm. The goal of the learner is to accumulate as much reward as possible (or equivalently achieve small regret), and it is well-known that algorithms such as UCB and $\eps$-greedy achieve optimal regret~\cite{auer2002finite}. } 

Given the fact that these well-known stochastic bandit algorithms are vulnerable, how can the learner defend against such an adversary? 
\cite{lykouris2018stochastic} initiated the design of bandit algorithms robust to adversarial corruptions under the \emph{weak} adversary model. While the strong adversary (the model in \cite{jun2018adversarial}) decides on the corruption after observing the learner's action and the realized reward, the weak adversary decides on the corruption before observing the learner's chosen action. Their work designed bandit algorithms with regret bounds that degrade gracefully as the corruption budget $C$ increases. The follow-up work \cite{gupta2019better, zimmert2021tsallis} further improved their regret bounds. However, when the learner faces a strong adversary, \cite{liu2019data}, \cite{he2022nearly} showed that any low-regret bandit algorithms can be hijacked (i.e. suffer linear regret) by only spending a sublinear attack cost. While it is apparent that the measure of regret is no longer a suitable benchmark when facing a \emph{strong} adversary in the stochastic bandit setting, is there some other benchmark in which robustness can be measured? This work proposes using the competitive ratio as the benchmark and gives two simple algorithms inspired by behavioral economics that achieve a competitive ratio arbitrarily close to 1. 

\subsection{Contributions}
\lbcomment{You might want to add some transitional sentences here to bridge introduction and contributions of your work, or expand the first sentence here a bit}
This work studies adversarial attacks in stochastic bandits under the \emph{strong} adversary model and studies both attack and defense strategies. The first set of results studies optimal attack strategies. It is shown that, when the learner employs the UCB and Thompson sampling algorithm, the adversary only needs to spend $O(\sqrt{\log T})$ attack cost to hijack the learner into choosing the target arm. This work further provides matching lower bounds and thus characterizes exactly the vulnerability of these algorithms under the adversarial attack scenario. This solves several open problems in \cite{jun2018adversarial}. \lbcomment{do you want to elaborate a bit more here?}

The second set of results studies defense strategies for the learner. While it is known no algorithm can achieve a sublinear regret in the presence of a \emph{strong} adversary (\cite{liu2019data}), this work studies the competitive ratio as a benchmark. I design two algorithms for the learner that achieve a competitive ratio arbitrarily close to 1. The algorithms draw connections from smoothed analysis and behavioral economics. I believe the use of competitive ratio under corruptions and the new connections merits further research. 




In the stochastic multi-arm bandit setting, there are $K$ arms, and arm $a\in[K]$ gives subgaussian rewards with mean $\mu_a$ and variance proxy $\sigma^2$. In the study of optimal attack strategies, the target arm the adversary wishes to promote is arm $K$, and denote $\Delta_a^+ = \max(0, \mu_a - \mu_K)$. The outline and main results of each section can be summarized as follows. 
\begin{itemize}
\item Section \ref{sec:prelim} begins with the problem statement. 
\item Section \ref{sec:UCBAttack} shows my design of an attack strategy against UCB with cost $\hatO(K\sigma\sqrt{\log T} + \sum_{a\neq K}\Delta_a^+)$. This holds for $T$ uniformly over time, improving the attack cost in~\cite{jun2018adversarial} by a $O(\sqrt{\log T})$ factor. 
\item Section \ref{sec:TSAttack} shows my design of an attack strategy against Thompson sampling with cost $\hatO(K\sqrt{\log T} + \sum_{a\neq K}\Delta_a^+ + K(\sigma + 1)\sqrt{ \log{K} } ) $. This holds for $T$ uniformly over time. To the best of the author's knowledge, no prior work explicitly studies adversarial attacks on Thompson sampling. 
\item Section \ref{sec:lowerBounds} proves lower bounds on the attack cost against UCB and Thompson sampling with order $\Omega(\sqrt{\log T})$. This shows the proposed attack strategy against UCB and Thompson sampling to be nearly optimal. A lower bound $\Omega(\log T)$ on the attack cost against $\eps$-greedy is also given. 
\item Section \ref{sec:smoothedDefense} presents the study of defense strategies and gives algorithms that achieve a competitive ratio arbitrarily close to 1, as long as the total corruption budget $C = o(T)$. 
\item Section \ref{section:experiments} describes numerical experiments. In the study of attack strategies, the experiments show significant improvements over~\cite{jun2018adversarial}. 
\end{itemize}

\szdelete{For the UCB and Thompson sampling algorithm, I design attack strategies with cost scaling as $\hatO(\sqrt{\log T})$ which holds for $T$ uniformly over time. Specifically, I design novel attack strategies against UCB and Thompson sampling that achieve optimal attack cost and provide matching lower bounds, hence resolving several open problems in~\cite{jun2018adversarial}. }

\szdelete{The main result is an attack strategy with cost $\hatO(K\sigma\sqrt{\log T} + \sum_{a\neq K}\Delta_a^+)$ which holds for $T$ uniformly over time, improving the attack cost in~\cite{jun2018adversarial} by a $O(\sqrt{\log T})$ factor. The attack manipulates a learner employing the UCB algorithm into pulling the target arm $T-o(T)$ times and succeeds with high probability. \szdelete{I also establish the first lower bound on the cumulative attack cost, in turn showing the attack to be near optimal. }\lbcomment{Better briefly describe what is the open problem you addressed in~\cite{jun2018adversarial}, your paper should better be self-contained that ideally readers don't have to read external references to understand the whole picture.} I also conduct numerical experiments to validate the theoretical results. The experiments also show a significant improvement over the attack strategy in~\cite{jun2018adversarial}.\lbcomment{I strongly recommend you use first-person perspective to describe the importance of your work here. Change it to something like ``I conducted numerical experiments to validate the theoretical results, which showed a significant improvement over the attack strategy in~\cite{jun2018adversarial}''} }

\szcomment{TODO, add motivating example}

\subsection{Related Works} 
The problem of adversarial attack on stochastic bandits was initiated by~\cite{jun2018adversarial}, in which they proposed attack strategies against UCB and $\eps$-greedy. The work by Liu and Shroff~\cite{liu2019data} studied black-box attacks against stochastic bandit algorithms. 
Recent works also studied adversarial attacks in other problems, including adversarial bandits (\cite{ma2023adversarial}), contextual bandits (\cite{ma2018data,garcelon2020adversarial}) gaussian process bandits (\cite{han2022adversarial}) and reinforcement learning (\cite{zhang2020adaptive}) etc. 

Another line of work studies the design of robust algorithms in the presence of adversarial corruption, sometimes under different adversary models. In the stochastic bandit setting, \cite{lykouris2018stochastic,gupta2019better} design learning algorithms for the learner robust to adversarial corruptions under the \emph{weak} adversary model, and their result was subsequently improved by \cite{zimmert2021tsallis}. The \emph{weak} adversary needs to decide on the corruption before observing the action of the learner, while the \emph{strong} adversary (the model in \cite{jun2018adversarial}) can observe the action of the learner before deciding on the corruption. The study of corruption-robust algorithms has also been studied in other problems, e.g. linear bandits (\cite{he2022nearly}), contextual search (\cite{leme2022corruption, zuo2023corruption}), and reinforcement learning (\cite{wei2022model}), to name a few. 

Smoothed analysis first appeared as a way to analyze an algorithm's performance beyond worst case, and was first introduced by \cite{spielman2004smoothed}. 
\lbcomment{rephrase as: \cite{spielman2004smoothed} first introduced smoothed analysis to analyze an algorithm's performance beyond worst case. BTW, what is an algorithm's performance}
In their paper, they showed that while the simplex algorithm exhibits exponential time complexity in the worst case, adding a Gaussian noise to the input guarantees a polynomial time complexity. The idea of smoothed analysis has also been explored in machine learning, online learning, as well as game-theoretic contexts (\cite{sivakumar2020structured, haghtalab2020smoothed, kannan2018smoothed}). \lbcomment{maybe add citation after each scenario, rather than put it at the end. like machine learning(some citation), online learning (some citation), game-theoretic context (some citation)}

Behavioral economics is mostly concerned with the study of bounded rationality (for a textbook treatment see e.g.\cite{camerer2011behavioral}). \lbcomment{what do you mean by textbook treatment}The quantal response was proposed in \cite{mckelvey1995quantal} as a solution concept when agents have bounded rationality. The quantal response enjoys many nice statistical properties; for example, it naturally arises from the logit model (\cite{mcfadden1976quantal, luce2012individual}) and is equivalent to selecting the maximum after a perturbation with Gumbel distribution (\cite{jang2016categorical}). The quantal response also appears in the machine learning literature under different contexts, e.g. the softmax activation function usually used in training machine learning models (\cite{dunne1997pairing}) and the multiplicative weight update algorithms in online learning (\cite{arora2012multiplicative}). Connections between behavioral economics and online learning has also been explored in (\cite{wu2022inverse}).

\szcomment{
Another related work is Liu and Shroff~\cite{liu2019data}, where they studied black-box attacks against stochastic bandit algorithms. 
Recent works also studied adversarial attacks on adversarial bandits~\cite{ma2023adversarial}, in which an adversary aims to manipulate the behavior of an adversarial bandit algorithm (e.g. EXP3). \cite{ma2018data,garcelon2020adversarial} studied adversarial attack on contextual bandits, in which the adversary manipulates the learner into choosing a target arm by modifying the reward or context vector. \cite{han2022adversarial} study adversarial attacks on gaussian bandits. Another line of work took the viewpoint of the learner and designed algorithms robust to adversarial corruptions~\cite{lykouris2018stochastic,gupta2019better}. 
}

%% file: sections/prelim.tex
\section{PRELIMINARIES}
\label{sec:prelim}

\begin{algorithm}[ht]
\caption{The general adversarial attack framework}
\label{attackframework}
\begin{algorithmic}
\FOR {$t = 1, 2, ...$} 
    \STATE Learner picks arm $a_t$ according to arm selection rule (e.g. UCB, Thompson sampling)\;
    \STATE Adversary learns $a_t$ and pre-attack reward $r^0_t$, chooses attack $\alpha_t$, suffers attack cost $|\alpha_t|$\;
    \STATE Learner receives reward $r_t = r_t^0 - \alpha_t$\;
\ENDFOR
\end{algorithmic}
\end{algorithm}

This work studies a stochastic multi-arm bandit problem where rewards are subject to adversarial corruptions. 
Let $T$ be the time horizon and $K$ the number of arms. The learner chooses arm $a_t \in [K]$ during round $t$, and a random reward $r_t^0$ is generated from a subgaussian distribution with variance proxy $\sigma^2$. The reward is centered at $\mu_{a_t}$:
\[
\Ex[r_t^0] = \mu_{a_t}. 
\]

The work studies the \emph{strong} adversary, who can observe the learner's chosen arm before deciding on the attack. At round $t$, after the learner chooses an arm $a_t$ and the reward $r_t^0$ is generated, but before the reward $r_t^0$ is given to the learner, the adversary adds a strategic corruption $\alpha_t$ to the reward $r_t^{0}$. Then the learner only receives the corrupted reward $r_t: = r_t^0 - \alpha_t$. Note that the adversary can decide the value of $\alpha_t$ based $(a_t, r_t^0)$ as well as the history $H_{t}$, where the history $H_t$ is defined as
\[
H_t = (a_1, r_{1}^0, \alpha_1, ..., a_{t-1}, r_{t-1}^0, \alpha_{t-1}). 
\]
The attack framework is summarized in~\cref{attackframework}. 

In the rest of this work, $\tau_a(t) := \{s : a_{s} = a, 1\le s < t \}$ denotes the set of timesteps that arm $a$ was chosen up to round $t$, and $N_a(t) := |\tau_a(t)|$ denotes the number of times arm $a$ has been pulled up until round $t$. Also let $\hat{\mu}_a(t)$ denote the post-attack empirical mean for arm $a$ in round $t$:
\[
\hat{\mu}_a(t) = \sum_{s\in \tau_a(t)} r_s / N_a(t),
\]
and let $\hat{\mu}^0_a(t)$ denote the pre-attack empirical mean for arm $a$ in round $t$:
\[
\hat{\mu}^0_a(t) = \sum_{s\in \tau_a(t)} r^0_s / N_a(t). 
\]

This work will study both attack strategies for the adversary and defense strategies for the learner. In the study of attack strategies, the goal of the adversary is to manipulate the learner into pulling some target arm $T-o(T)$ times, while minimizing cumulative attack cost, defined as $\sum_{t=1}^T |\alpha_t|$. Without loss of generality, this work assumes the target arm is $K$. In the study of defense strategies, the goal of the learner is to optimize the cumulative reward, while being agnostic to the attack strategy of the adversary. 



\subsection{A Concentration Result}
\lbcomment{this title sounds weird to me}
The following concentration result will be useful throughout the analysis. Set parameter $\beta(n)$ as:
\[
\beta(n) = \sqrt{ \frac{2\sigma^2}{n} \log\frac{\pi^2 K n^2}{3\delta}}, 
\]
and define event $E$ as
\[ 
\forall a, t, |\hat{\mu}^0_a(t) - \mu_a| < \beta(N_a(t))
\]
which represents the event that pre-attack empirical means are concentrated around the true mean within an error of $\beta(N_a(t))$. The following has been shown in~\cite{jun2018adversarial}, which follows from a Hoeffding inequality combined with a union bound. 
\begin{lemma}[\cite{jun2018adversarial}]
\label{lemma:eventE}
Event $E$ happens with probability $1- \delta$. Further, the sequence $\beta(n)$ is non-increasing in $n$. 
\end{lemma}

%% file: sections/UCBattack.tex
\section{\uppercase{Attack Strategy Against UCB}}
\label{sec:UCBAttack}

In this section, I first give the specification of the UCB algorithm (\cref{algo:ucb}), then propose a near-optimal attack strategy against it (\cref{algo:ucb_optimal_attack}). 

\begin{algorithm}[ht]
\caption{UCB, adapted from \\
\cite{bubeck2012regret}}
\label{algo:ucb}
\begin{algorithmic}
\State For each arm $a$, learner maintains empirical mean $\hatmu_a(t)$ and the number of times $a$ has been pulled $N_a(t)$
\FOR {$t = 1, 2, \dots, K$} 
\STATE Pull each arm $a$ once\;
\STATE Update $\hatmu_a(t+1), N_a(t+1)$
\ENDFOR
\FOR {$t > K$} 
\STATE $a_t = \arg\max_{a} \hat{\mu}_a(t) + 3\sigma\sqrt{\frac{\log t}{N_a(t)}}$\;
\STATE Choose arm $a_t$ and observe reward
\State Update $\hat{\mu}_{a}(t+1), N_{a}(t+1)$\;
\ENDFOR
\end{algorithmic}
\end{algorithm}

\begin{algorithm}[ht]
\caption{Optimal Attack on UCB}
\label{algo:ucb_optimal_attack}
\begin{algorithmic}
\STATE $\beta(n) = \sqrt{\frac{2\sigma^2}{n}\log\frac{\pi^2 K n^2}{3\delta}}$\;
\FOR {$t = 1, 2, \dots$} 
    \IF {$a_t \neq K$} 
        \STATE Adversary observe reward $r_t^0$
        \STATE Compute and inject the smallest corruption $\alpha_t$ with $\alpha_t \ge 0$, such that:
        \[
            \hat{\mu}_{a_t}(t) \le \hat{\mu}_K(t) - 2\beta(N_K(t)) - 3\sigma \cdot \exp(N_{a_t}(t))
        \]
        where $\hat{\mu}_{a_t}(t)$ is the post-attack empirical mean:  
        \[
        \hat{\mu}_{a_t}(t) = (\hat{\mu}_{a_t}(t-1)\cdot N_a(t) + r_t^0  - \alpha_t) / (N_a(t)+1)
        \]
            
    \ENDIF
\ENDFOR
\end{algorithmic}
\end{algorithm}

\subsection{UCB}
The UCB algorithm is summarized in~\cref{algo:ucb} and the specification follows from~\cite{jun2018adversarial}. In the first $K$ rounds, the learner pulls each arm $a$ once to obtain an initial estimate $\hat{\mu}_a$. 
Then in later rounds $t > K$, the learner computes the UCB index for arm $a$ as
\[
    \hat\mu_a(t) + 3\sigma\sqrt{\frac{\log t}{N_a(t)}}. 
\]
The arm with the largest index is then chosen by the learner.

\subsection{Adversarial Attack Strategy}
I now show an optimal attack strategy for the adversary against the UCB algorithm that only spends $\hatO(\sqrt{\log T})$ attack cost. The attack strategy is summarized in~\cref{algo:ucb_optimal_attack}. Recall the goal of the adversary is to manipulate a learner employing the UCB algorithm into choosing the target arm (arm $K$) at least $T- o(T)$ times while keeping the cumulative attack cost low.  

For convenience assume arm $K$ is picked in the first round. \lbcomment{This sentence sounds awkward to me.} The proposed attack strategy works as follows. The adversary only attacks when any non-target arm is pulled, and adds corruption to ensure the difference between the post-attack empirical mean of the pulled arm and the target arm is above a certain gap. Specifically, the attacker ensures that the post-attack empirical means satisfy:
\begin{align}
\hat{\mu}_{a_t}(t) \le \hat{\mu}_K(t) - 2\beta(N_K(t)) - 3\sigma \exp(N_a(t)). 
\label{eq:attackGapUCB}
\end{align}

The key insight is that in order to minimize attack cost, the number of non-target arm pulls should be kept as low as possible. In fact, using the proposed attack strategy guarantees any non-target arm is pulled only $O(\log\log t)$ times for any round $t$. \szcomment{add details}

The main result on the upper bound of the cost of the attack strategy against UCB is given below. Recall $\Delta_a = \mu_a - \mu_K$, $\Delta_a^+ = \max(0, \Delta_a)$.
\begin{theorem}
\label{thm:UCBAttack}
With probability $1- \delta$, for any $T$, using the proposed attack strategy ensures any non-target arm is pulled $O(\log \log T)$ times and total attack cost is $\widehat{O}(K\sigma\sqrt{\log T} + \sum_{a\in [K]}\Delta_a^+)$. 
\end{theorem}

\begin{proof}[Proof Sketch]
Each time a non-target arm is pulled, the adversary injects corruption so that the gap in~\cref{eq:attackGapUCB} holds. The gap $2\beta(N_K(t)) + 3\sigma \exp(N_a(t))$ consists of two terms. The first term $\beta(N_K(t))$ is a deviation bound that accounts for the estimation error of the true means (see~\cref{lemma:eventE}). The second term grows exponentially with the number of times the current arm is pulled and guarantees that any non-target arm is only pulled for $0.5\log\log T$ times for any round $T$. This in turn implies for any round $t$, the adversary needs only spend $\hatO(\exp(0.5 \log\log T)) = \hatO(\sqrt{\log T})$ attack cost to ensure the gap holds. 
\end{proof}

\begin{remark}
Note that in the actual implementation, the adversary may wish equation~\cref{eq:attackGapUCB} to hold with strict inequality. This can be accomplished by adjusting the attack by an infinitesimal amount. This work will not be concerned with such an issue and simply assumes~\cref {eq:attackGapUCB} holds with equality when the adversary attacks. 
\end{remark}

%% file: sections/thompson.tex
\section{\MakeUppercase{Attack Strategy Against Thompson Sampling}}
\label{sec:TSAttack}

In this section, I first give the description of the Thompson sampling algorithm (\cref{algo:TS}), then propose a near-optimal attack strategy against it (\cref{algo:attackTS}). 
\begin{algorithm}[h]
\caption{Thompson sampling, adapted from \cite{agrawal2017near}}
\label{algo:TS}
    \begin{algorithmic}
        \State For each arm $a$, learner maintains empirical mean $\hatmu_a(t)$ and the number of times $a$ has been pulled $N_a(t)$
        \FOR {$t = 1, 2, \dots, K$}
            \State Pull each arm $a$ once
            \State Update $\hatmu_a(t+1), N_a(t+1)$
        \ENDFOR
        \FOR {$t > K$}
            \State Sample $\nu_a = \cN(\hatmu_a(t), \frac{1}{N_a(t)})$
            \State Choose $a_t = \arg\max_a \nu_a$ and observe reward
            \State Update $\hatmu_{a}(t+1), N_{a}(t+1)$
        \ENDFOR
    \end{algorithmic}
\end{algorithm}

\begin{algorithm}[h]
\caption{Optimal Attack on Thompson Sampling}
\label{algo:attackTS}
\begin{algorithmic}
\STATE $\beta(n) = \sqrt{\frac{2\sigma^2}{n}\log\frac{\pi^2 K n^2}{3\delta}}$\;
\FOR{$t = 1, 2, \dots$}
    \IF {$a_t \neq K$}
        \State Compute and inject the smallest corruption $\alpha_t$ with $\alpha_t \ge 0$, such that:
        \[\hatmu_{a_t}(t)\le\hatmu_{K} - 2\beta(N_K(t)) - 4\exp(N_{a_t}(t)) - \sqrt{8 \log\frac{\pi^2 K}{3\delta}}
        \]
        where $\hat{\mu}_{a_t}(t)$ is the post-attack empirical mean:  
        \[
        \hat{\mu}_{a_t}(t) = (\hat{\mu}_{a_t}(t-1)\cdot N_a(t) + r_t^0  - \alpha_t) / (N_a(t)+1)
        \]
    \ENDIF
\ENDFOR
\end{algorithmic}
\end{algorithm}

\subsection{Thompson Sampling}
The Thompson Sampling is summarized in~\cref{algo:TS} and the specification is adapted from~\cite{agrawal2017near}. In the first $K$ rounds, the learner pulls each arm $a$ once. Then in later rounds, for each arm $a$, a random variable $\nu_a$ is generated from the distribution $\cN(\hatmu_a(t), \frac{1}{N_a(t)})$. The arm with the largest $\nu_a$ is then chosen.

\subsection{Adversarial Attack Strategy}
I now show an optimal attack strategy for the adversary against the Thompson sampling algorithm, which only spends $\hatO(\sqrt{\log T} )$ attack cost. The attack strategy is summarized in~\cref{algo:attackTS} and shares some similar insights as the attack strategy against the UCB algorithm. Specifically, any non-target arm pulls are also upper bounded by $O(\log\log t)$. This is achieved by injecting corruption whenever non-target arms are pulled to ensure the post-attack empirical means satisfy: 
\[
\hatmu_{a_t}(t)\le\hatmu_{K}(t) - 2\beta(N_K(t)) - 4 \exp(N_{a_t}(t)) - \sqrt{ 8 \log\frac{\pi^2 K}{3\delta}}.  
\]
Similar to the attack on UCB, the gap contains a term $4\exp( N_{a_t}(t) )$ that grows exponentially in the number of times that the non-target arm is pulled. Hence, this exponential term exhibits some generality and may act as a general attack principle for the adversary. 

\begin{theorem}
\label{thm:TSAttack}
With probability $1 - 2\delta$, for any $T$, using the proposed attack strategy ensures any non-target arm is pulled $O(\log\log T)$ times and the total attack cost is $\hatO(K\sqrt{\log T} + \sum_{a\neq [K]} \Delta_a^+ + K (\sigma+1)\sqrt{\log\frac{\pi^2 K}{3\delta} })$. 
\end{theorem}

\begin{remark}
Note that in the attack cost $\hatO(K\sigma\sqrt{\log T} + \sum_{a\neq K}\Delta_a^+)$ on UCB bandits, the $\sqrt{\log T}$ factor is multiplied by $\sigma$, whereas for Thompson sampling the $\sqrt{\log T}$ factor is not. Hence, the attack cost for UCB is expected to grow faster than the attack cost for Thompson sampling as $\sigma$ gets large, and vice versa. This point is in fact illustrated in the numerical experiments in the following sections. 
\end{remark}

%% file: sections/lowerBound.tex
\section{\uppercase{Lower Bounds on Attack Cost}}
\label{sec:lowerBounds}
 

In this section, I prove lower bounds on the cumulative attack cost. For a learner employing the UCB or Thompson sampling algorithm, the lower bounds of $\Omega(\sqrt{\log T})$ match the upper bound in the previous section up to $O(\log \log T)$ factors, showing the proposed attack strategy to be near optimal. I also show a lower bound of $\Omega(\log T)$ on the attack cost against the $\varepsilon$-greedy algorithm. I shall focus on the setting where $K = 2$, but the results also generalize to the case where $K > 2$. In this section, the bandit environment consists of two arms giving Gaussian rewards $\cN(\mu_1, \sigma^2), \cN(\mu_2, \sigma^2)$. The mean $\mu_1 > \mu_2$ and the 2nd arm is the target arm. Let $\Delta = \mu_1 - \mu_2$ and let $B$ be a sufficiently large constant. 

The below two theorems characterize the lower bound on attack cost against UCB and Thompson sampling. 
\begin{theorem} [UCB Attack Cost Lower Bound]
\label{thm:UCBLowerBound}
Assume the learner is using the UCB algorithm as in~\cref{algo:ucb}. 
For any $T > B$, if the adversary spend an attack cost less than $\Delta + 0.22\sigma\sqrt{\log {T} - 1}$, then with probability 0.9, the learner pulls the first arm more than $T / 26$ times. 
\end{theorem}
\begin{proof}[Proof Sketch]
Consider the last time the target arm has been pulled, denote this round by $t_0$. At round $t_0$, the UCB index for the non-target arm (the optimal arm) must be lower than that of the target arm. Since $N_2(t_0) = \Theta(T)$:
\[
\hatmu_1(t_0) + \sqrt{\frac{\log T}{N_1(t_0)}} \lesssim \hatmu_2(t_0). 
\]
To drag down the UCB index of the non-target arm to achieve this, an attack cost of 
\[
\left( \Delta + \sqrt{\frac{\log T}{N_1(t_0)}} \right) N_1(t_0) = \Omega(\sqrt{\log T})
\]
is needed. 
\end{proof}

\szcomment{ This can be further lower bounded by:
\[
\frac{1}{B}\exp(\frac{-B^2}{2}) = \frac{1}{0.1\sqrt{\log T}} T^{0.01} > T^{0.005}
\]
}

\szdelete{
\begin{lemma}
Consider rounds $[T/200, T / 2]$. With probability $0.5$, the 1st arm has been chosen $xxx$ times during this period. 
\end{lemma}

\begin{lemma}
Consider rounds $[T/2, T]$. With probability 0.5, the 1st arm has been chosen $xxx$ times during this period. 
\end{lemma}

\begin{proof}
\begin{align*}
\Pr[a_t = 1] &> \Pr[\nu(t) > 0] \\
&= \Pr[ \cN (\hatmu_1(t) - \hatmu_2(t), \frac{1}{N_1(t)} + \frac{1}{N_2(t)} ) > 0 ]\\
&> \Pr[ \cN( \frac{- B}{N_1(t_0)}, \frac{1}{N_1(t_0)} ) > 0 ] \\
&> \Pr[ \cN(\frac{-B}{T^{0.6}}, \frac{1}{T^{0.6}} ) > 0 ]\\
&> \Pr [  ]
\end{align*}
\end{proof}
}

\begin{theorem} [Thompson Sampling Attack Cost Lower Bound]
\label{thm:TSLowerBound}
Assume the learner is using the Thompson Sampling algorithm as in~\cref{algo:TS}. 
For any $T > B$, if the adversary spent an attack cost no more than $\Delta + 0.1\sqrt{\log T}$, then with probability 0.81, the learner pulls the 1st arm more than $T / 10$ times. 
\end{theorem}

\begin{remark}
Note the lower bounds hold for when the adversary knows the time horizon $T$, whereas the upper bounds in the previous section hold for $T$ uniformly over time. Therefore, it would seem the $O(\log\log T)$ factor is the price the adversary has to pay when moving from the fixed $T$ setting to the uniform $T$ setting. 
\end{remark}

\szcomment{
Other Ideas. 
Consider indicator variable. 
Consider the last round and argue the non-target got a bonus. 

At round $t$, consider the gaussian random variable from distribution $N(\hatmu_2 - \hatmu_1, \frac{1}{N_2} + \frac{1}{N_1})$:
\[
\nu(t) = \nu_2(t) - \nu_1(t)
\]
the probability that the non-target arm gets chosen is $Pr[\nu(t) < 0]$. This can be upper bounded by (by increasing the variance): $Pr[N(0, \frac{2}{N_1}) > \hatmu_2 - \hatmu_1]$. Of all rounds $t$ during which target-arm is chosen, let $xxx$ be the round such that $\hatmu_2 - \hatmu_1$ is maximized. 
}

For comparison, I establish a lower bound on the cumulative attack cost against $\varepsilon$-greedy. Interestingly, the $\eps$-greedy exhibit a $\Omega(\log T)$ lower bound, in contrast with the $\Omega(\sqrt{\log T})$ lower bounds for UCB and Thompson sampling. This lower bound shows the attack strategy proposed in~\cite{jun2018adversarial} to be essentially optimal. 

The $\varepsilon$-greedy algorithm works as follows. At each round, the learner with probability $\varepsilon_t$ does uniform exploration, otherwise, the learner does exploitation and chooses the arm with the largest empirical reward. Assume $\varepsilon_t = \nicefrac{cK}{t}$ for some exploration parameter $c$ as in~\cite{auer2002finite} (each arm is chosen for exploration with probability $\nicefrac{c}{t}$ each round). 


\begin{theorem} [$\eps$-greedy Attack Cost Lower Bound]
\label{thm:eps-greedyLowerBound}
Assume the learner is using the $\varepsilon$-greedy algorithm with a learning rate $2c/t$ for some fixed constant $c$. For any $T > B$, if the adversary spent an attack cost no more than $c\cdot \Delta \log T / 6$, then with probability $0.8$, the learner pulls the 1st arm more than $T/ 4$ times.
\end{theorem}

%% file: sections/smoothed.tex
\section{DEFENSES WITH SMOOTHED RESPONSES}
\label{sec:smoothedDefense}
This section studies defense strategies for the learner, and uses the competitive ratio as the benchmark instead of regret. This is because a sublinear regret is not possible in the presence of a strong adversary. Specifically, it is known for any low-regret bandit algorithm, a strong adversary can spend $o(T)$ attack cost and make the learner suffer linear regret (\cite{he2022nearly}). 

This section will prove bounds on the collected rewards of the form
\begin{align*}
\REW \ge (1 - \eps)\OPT - o(T), 
\end{align*}
where $\REW$ is the total reward collected, as measured by the true empirical means, and $\OPT$ is the expected reward of the optimal policy, i.e.:
\[
\REW = \sum_{t=1}^T \mu_{a_t}, \quad \OPT = T\mu^*. 
\] If an algorithm satisfies the above lower bound on $\REW$ (with probability $1 - \delta$), then the algorithm is said to achieve $(1 - \eps$) competitive ratio (with probability $1 - \delta$). In this section, pre-attack rewards and post-attack rewards are assumed to be bounded in $[0,1]$. This assumption is made since the competitive ratio is used as a benchmark; however, note even in this bounded rewards setting, \cite{rangi2022saving} showed that no-regret bandit algorithms are prone to adversarial attacks.

I first begin with a more detailed discussion on the use of competitive ratio as benchmark. Then, I show how concepts from smoothed analysis and behavioral economics can be used to design algorithms that achieve a competitive ratio arbitrarily close to 1 in the presence of a strong adversary. 

\subsection{Discussion on the Use of Competitive Ratio}
The strongest form of guarantee one can hope for is a sublinear regret that scales with $C$. This is possible when the learner faces a \emph{weak} adversary. However, obtaining a sublinear regret is impossible when the adversary is \emph{strong}. 
\begin{fact} (From \cite{he2022nearly})
\label{fact1}
If some algorithm $\ALG$ achieves regret $\Reg(T)$ when $C = 0$, then there exists a strong adversary with budget $C = \Theta(\Reg(T))$ that can make the learner suffer linear regret. 
\end{fact}
From the above fact, if the algorithm were not to suffer linear regret when $C=0$, then there exists a scenario in which the adversary uses a sublinear attack budget and makes the algorithm suffer linear regret. In either case, there is a scenario where the learner must suffer linear regret. Now, since sublinear regret is not possible, the next natural possible benchmark is the competitive ratio. To understand how good this benchmark is, consider the following proposition. 
\begin{prop}
Fix any algorithm $\ALG$. There exists some constant $\eps > 0$ and an attack strategy with a sublinear budget such that the learner achieves no better than $1- \eps$ competitive ratio, specifically,
\[
\lim \inf_{T\rightarrow \infty} ({\REW(T)} / {\OPT(T)}) \le 1 - \eps. 
\]
\end{prop}
This can be seen as a consequence of Fact~\ref{fact1} above. For, considering first when $C = 0$, if the above is not satisfied for any $\eps > 0$ (otherwise we already have our $\eps$), then the conclusion must be $\ALG$ achieves sublinear regret when $C = 0$. Then applying the above Fact~\ref{fact1} gives a suitable attack strategy with a suitable $\eps$. Though using the competitive ratio as benchmark gives us non-trivial robustness guarantees, it is unclear whether there exists a better benchmark under which robustness can be measured. 

\subsection{Smoothed Myopic Response}
Motivated by smoothed analysis, this work proposes the smoothed myopic response as a defense strategy. At each round, let the empirically best arm be $a_t^*$. The response is then a $\rho$-smoothed version of the myopic response. In other words, every arm $a\neq a_t^*$ is pulled with probability $\rho$, and the empirical best arm is pulled with probability $1-(K-1)\rho$. When the learner puts a small constant as exploration probability on each arm, he will eventually discover the best arm as long as the total corruption budget of the adversary is sublinear in $T$. 
\begin{algorithm}[t]
\caption{Smoothed Response}
\label{alg:smoothMyopic}
\label{alg:smoothedResponse}
\begin{algorithmic}
\STATE Input: $\eps$, target competitive ratio is $(1 - \eps)$
\STATE $\rho := \eps / K$
\FOR {$t = 1, 2, \dots, T$}
\STATE Let $a_t^* = \arg\max \hat{\mu}_{a}(t)$
\STATE Let
\begin{equation*}
p_{t,a} = \begin{cases}
\rho &\text{if $a \neq a_t^*$}\\
1 - (K-1)\rho &\text{if $a = a_t^*$}
\end{cases}
\end{equation*}
\STATE Choose arm sampled from $p_t$
\ENDFOR
\end{algorithmic}
\end{algorithm}

\begin{theorem}
\label{thm:smoothMyopic}
Fix any $\eps > 0$, for a sufficiently large $T$, with probability $1- 1/T$, \cref{alg:smoothMyopic} achieves the following:
\[
\REW \ge (1 - \eps)\OPT - o(T). 
\]
\end{theorem}
\begin{proof}[Proof Sketch]
Let $a$ be any suboptimal arm. At any round, arm $a$ has a probability of at least $\rho$ of being chosen. After $T = \Omega(\frac{C}{\rho\Delta_a})$ rounds, arm $a$ has been chosen at least $\Omega(\frac{C}{\Delta_a})$ times and the effect of the corruption diminishes. Specifically, the gap between pre-attack mean and post-attack mean drops below $O(\Delta_a)$. Consequently, after enough rounds (which is sublinear in $T$), the suboptimality is identified and the learner will choose arm $a$ with probability $\rho$. Setting $\rho$ to be a sufficiently small constant achieves a competitive ratio arbitrarily close to 1. 
\end{proof}

\subsection{Quantal Response}
The second response model is motivated by literature on bounded rationality, specifically the quantal response model. At each round, the learner computes the empirical mean and the ratio to the empirically best arm $\psi_{a}(t) = \hatmu_a(t) / \hatmu^*(t)$. The learner then assigns a probability to pull each arm proportional to $\exp(\lambda \psi_a(t))$, where $\lambda$ is a parameter chosen by the learner. Thus, the probability can be interpreted as performing a softmax on the empirical means. In addition, the parameter $\lambda$ controls the `sharpness' of the smoothed distribution, the larger the $\lambda$, the less exploration the learner takes. To illustrate this point, if $\lambda$ is taken to be $+\infty$, the learner always acts myopically, and if $\lambda = 0$, the learner always does uniform exploration. 
\begin{algorithm}[t]
\caption{Quantal Response}
\label{alg:quantalRResponse}
\begin{algorithmic}
\State Input: $\eps$: target competitive ratio is $(1 - \eps)$; 
\State $\lambda := 2\ln\frac{K}{\varepsilon}$
\FOR {$t = 1, 2, \dots, T$}
\State $\hatmu^*(t) = \arg\max_a \hatmu_a(t)$
\STATE $\psi_{a}(t) = {\hatmu_a(t)} / {\hatmu^*(t)}$
\STATE Let $p_{t,a} = { 
\exp(\lambda \psi_{a}(t) )
} / {
\sum_{b} \exp(\lambda \psi_b(t) )
} $
\STATE Choose arm sampled from $p_t$
\ENDFOR
\end{algorithmic}
\end{algorithm}

\begin{theorem}
\label{thm:quantal}
Fix any constant $\eps>0$, for a sufficiently large $T$, with probability $1-1/T$, \cref{alg:quantalRResponse} achieves:
\[
\REW \ge (1 - \eps)\OPT - o(T). 
\]
\end{theorem}

%% file: sections/experiments.tex
\section{EXPERIMENTS}
\label{section:experiments}

This section describes the numerical simulations on the proposed optimal attack strategies and defense strategies \footnote{Code available at \url{https://github.com/ShiliangZuo/BanditAttack.git}}. 

\begin{table}[t]
\centering
\caption{Comparisons of cumulative attack cost. The first entry `Baseline' runs attack strategy in~\cite{jun2018adversarial} against UCB. Second and third entry runs the proposed attack strategy against UCB and Thompson Sampling (TS). Results on proposed attack strategy all have standard deviation within 1.0 after 10 random trials. }
\label{table:exp}
\vspace{0.1in}
\begin{tabular}{l|l|lll}
\hline
Setting                  & $\sigma$ & $\mu = 0.1$ & $\mu = 1$ & $\mu = 2$ \\ \hline
Baseline          & 0.1      & 23.6        & 129.4     & 247.3     \\
                  & 1        & 114.4       & 241.7     & 360.3     \\
                  & 2        & 239.4       & 367.6     & 475.5     \\ \hline
UCB               & 0.1      & 1.3         & 2.4       & 3.6      \\
                  & 1        & 14.5        & 15.9      & 16.8    \\
                  & 2        & 30.3        & 30.7      & 31.0    \\ \hline
TS & 0.1      & 13.0        & 13.9     & 15.0     \\
                  & 1        & 19.0       & 19.7      & 20.6     \\
                  & 2        & 23.8        & 25.0     & 26.7     \\ \hline
\end{tabular}
\end{table}
\input{sections/Figure}

\subsection{Experiments on Attack Strategies}

In this subsection, I simulate the proposed attack strategies on UCB and Thompson sampling and describe the results of numerical experiments. In the experiments, the bandit instance has two arms, and the reward distributions are $\cN(\mu, \sigma^2)$ and $\cN(0, \sigma^2)$ respectively. The target arm is the second arm. The experiments aim to empirically study how the variance of the reward $\sigma^2$ and the reward gap $\mu$ affect the cumulative attack cost. For both UCB and Thompson sampling, I conduct 9 groups of experiments by varying the parameters of $\sigma \in \set{0.1, 1, 2}$ and $\mu \in \set{0.1, 1, 2}$. In each group, I run 20 trials for the bandit instance with $T = 10^6$. I also run the attack strategy against UCB in~\cite{jun2018adversarial} as a baseline for comparison. 


In the experiments for both UCB and Thompson sampling, any non-target arm is pulled no more than 2 times in any trial, while the target arm is pulled almost every round. This validates the theoretical results, which indicate that any non-target arm gets pulled no more than $0.5\log\log T$ times.

\szdelete{
For both UCB and Thompson sampling, the theoretical results in this work indicate non-target arm is pulled at most $0.5\log\log T$ times; the empirical results validate this, as the non-target arm is pulled no more than 2 times in any trial, while the target arm is pulled almost every round. \lbcomment{chatgpt version: The empirical results validate this, as the non-target is pulled no more than 2 times in any trial, while the target arm is pulled almost every time.}\szcomment{updated}
}

The cumulative attack costs for different choices of $(\mu, \sigma)$ are summarized in~\Cref{table:exp}. For UCB, the empirical results fit nicely with the theoretical bound of $\hatO(\sigma\sqrt{\log T} + \mu)$ in this work (specializing the upper bound on attack cost on UCB to the 2 arm setting, similar for the following bounds). \lbcomment{Ask chatgpt to help you summarize your sentences here; every sentence here is too short, and you should somehow combine them to make them more interesting to read} The results also show a significant improvement over the attack strategy proposed in~\cite{jun2018adversarial}, which had a theoretical bound of $O(\mu\log T + \sigma\log T)$. For Thompson sampling, the empirical results fit nicely with the theoretical bound of $\hatO(\sqrt{\log T} + \mu + \sigma)$. Also note that in the attack cost on UCB bandits, the $\sqrt{\log T}$ factor is multiplied by $\sigma$ but not for Thompson sampling. \lbcomment{Also note that $\sqrt{\log T}$ factor is multiplied by $\sigma$ in the attack cost on UCB bandits but not for Thompson sampling.} Thus, the attack cost for UCB is expected to grow faster than the attack cost for Thompson sampling as $\sigma$ gets large, and vice versa, as shown is~\Cref{table:exp}. \lbcomment{delete last sentence} 

The experiments validate the theoretical results and empirically demonstrate that adding very small corruptions allows the adversary to steer the learner away from the actual optimal arm and manipulates the learner into choosing the adversary's target arm. \lbcomment{Hence, the numerical experiments validate the theoretical results and empirically demonstrate that adding a small corruption allows the adversary to manipulate the learner's choice away from the true optimal option, prompting the learner to select the adversary's desired target arm.
}

\subsection{Experiments on Defense Strategies}

This subsection describes the second set of experiments on the two proposed defense strategies. There are $2$ arms giving Bernoulli rewards. The first arm is the optimal arm and gives rewards with mean $\mu_1 = 0.5$, the second arm is the suboptimal arm and gives rewards with mean $\mu_2 = 0.2$. The time horizon $T = 10^5$ and the corruption budget of the adversary is $C = 10^3$. The adversary adopts the following attack strategy: whenever the learner chooses the first arm (i.e. the optimal arm) and the realized reward is 1, the adversary changes this reward to 0. Figure~\ref{figure:smoothedMyopic} and figure~\ref{figure:quantal} show the cumulative reward as a function of $T$, using the smoothed myopic response and the quantal response respectively. In general, when the learner does more exploration (larger $\rho$ in smoothed myopic response and smaller $\lambda$ in quantal response), he will be quicker to identify the optimal arm. But this excess exploration will hurt performance in the long run.

%% file: sections/Figure.tex
\begin{figure}[!t]
\centering
\includegraphics[width=0.5\textwidth]{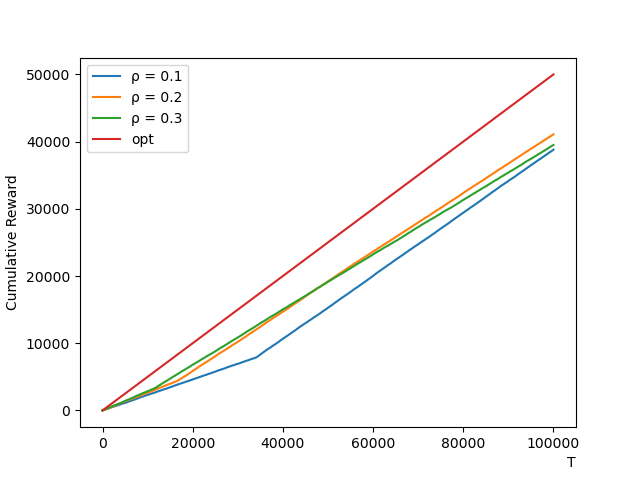}
\includegraphics[width=0.5\textwidth]{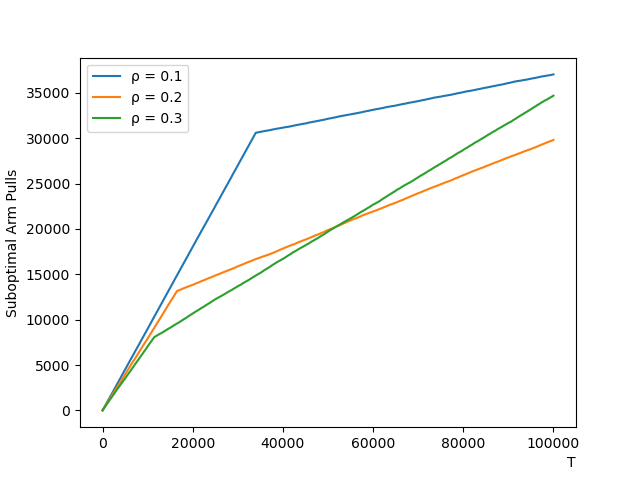}
\caption{Top subfigure shows cumulative reward under smoothed myopic response, bottom subfigure shows the number of suboptimal arm pulls. The learner does more exploration as $\rho$ increases and will be quicker to identify the optimal arm, but the excess exploration hurts performance in the long run. All values have coefficient of deviation within 0.02 after 10 random trials. }
\label{figure:smoothedMyopic}
\end{figure}
\begin{figure}[!t]
\centering
\includegraphics[width=0.5\textwidth]{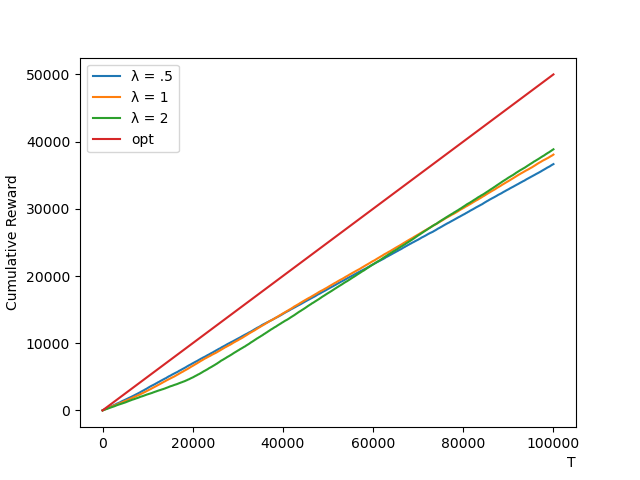}
\includegraphics[width=0.5\textwidth]{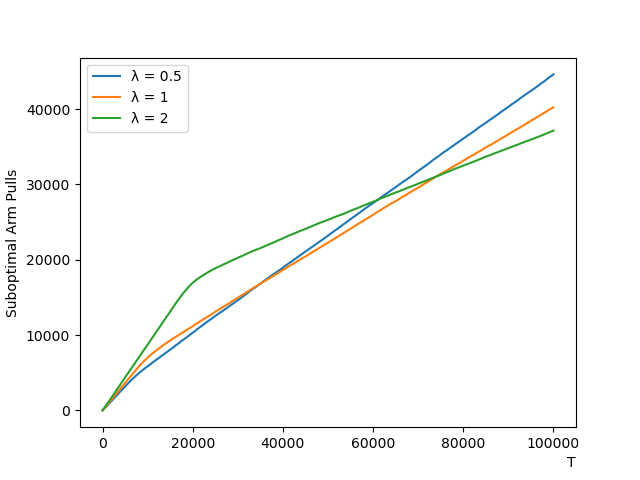}
\caption{Top subfigure shows cumulative reward under quantal response, bottom subfigure shows the number of suboptimal arm pulls. The learner does more exploration as $\lambda$ decreases and will be quicker to detect the optimal arm, but the excess exploration hurts performance in the long run. All values have coefficient of deviation within 0.02 after 10 random trials. }
\label{figure:quantal}
\end{figure}

%% file: sections/conclusion.tex
\section{CONCLUSION}
In this work, I study adversarial attacks that manipulate the behavior of stochastic bandit algorithms by corrupting the reward the learner observes. Both attack and defense strategies are studied. From the adversary's perspective, I give nearly optimal attack strategies. Tight characterizations are given on the attack cost needed to manipulate the UCB, Thompson sampling, and $\eps$-greedy algorithm into pulling some target arm of the adversary's choosing. For UCB and Thompson sampling, I propose optimal attack strategies with attack cost $\hatO(\sqrt{\log T})$ and establish matching lower bounds $\Omega(\sqrt{\log T})$ on the cumulative attack cost. For the $\eps$-greedy algorithm, I give a lower bound of $\Omega(\log T)$ on the attack cost. I also stdudy defense strategies for the learner. Motivated by literature from smoothed analysis and behavioral economics, I give two simple algorithms that achieve a competitive ratio arbitrarily close to 1. 
\lbcomment{Conclusion should be using present tense.}

\lbcomment{Since you mention back epsilon greedy algorithm at the conclusion. you might also want to add a couple of sentences to describe this algorithm compared with UCB. You simply mentioned the name in the Introduction part, which lacks some of the connection in the conclusion.}

%% file: appendix/proofs_UCB.tex
\section{\MakeUppercase{Proofs: Attack Strategy against UCB}}
This section details the proof of~\Cref{thm:UCBAttack}. I begin with a lemma characterizing the number of pulls for any non-target arm. Recall event $E$ is the event that pre-attack empirical means concentrate around the true mean (\Cref{lemma:eventE}). 

\begin{lemma}
\label{lemma:UCBpulls}
Assume event $E$ holds. At any round $t$, $N_a(t) \le \ceil{0.5\cdot\log{\log t}}$ for any $a \neq K$. 
\label{lemma:loglog-nontarget}
\end{lemma}

\begin{proof}
For sake of contradiction suppose some non-target arm $a$ is pulled more than $\ceil{0.5\cdot\log{\log t}}$ times. After this arm is pulled for the $\ceil{0.5\cdot\log{\log t}}$-th time at round $t_0 < t$, we must have 
\begin{align}
\hat{\mu}_a(t_0) &\le \hat{\mu}_K(t_0) -2\beta(N_K(t_0)) - 3\sigma \cdot \exp(\log\sqrt{\log {t}}) \nonumber\\
&= \hat{\mu}_K(t_0) -2\beta(N_K(t_0)) - 3\sigma \sqrt{\log {t}}. 
\label{eq:attack1}
\end{align}

Now assume arm $a$ has been pulled for the $(\ceil{0.5\log{\log t}} + 1)$-th time in round $t_1 \in [t_0+1,t]$. Then the UCB index of arm $a$ must be higher than that of arm $K$ in round $t_1$. However,
\begin{align*}
    &\hat{\mu}_a(t_1-1) + 3\sigma\sqrt{\frac{\log t_1}{N_{a}(t_1-1)}} \\
    &= \hat{\mu}_a(t_0) + 3\sigma\sqrt{\frac{\log t_1}{N_{a}(t_0)}}\\
    &\le \hat{\mu}_K(t_0) - 2\beta(N_K(t_0)) - 3\sigma \sqrt{\log t} + 3\sigma\sqrt{\frac{\log t_1}{N_a(t_0)}}\\
    &\le \hat{\mu}_K(t_1) -3\sigma\sqrt{\log t} + 3\sigma\sqrt{\frac{\log t_1}{N_a(t_0)}} \\
    &\le \hat{\mu}_K(t_1). 
\end{align*}
The second line follows from the fact that arm $a$ has not been chosen since $t_0$, the third line follows from the design of the attack strategy (specifically~\cref{eq:attack1}), and the fourth line follows from the concentration result given by event $E$. 
The UCB index of arm $a$ is lower than that of arm $K$, hence a contradiction is established, and arm $a$ will not be picked again. 
\end{proof}

\proof (of~\Cref{thm:UCBAttack})
Assume event $E$ holds throughout this proof. By~\cref{lemma:loglog-nontarget}, any non-target arm is pulled $O(\log\log T)$ times. 
Recall $\tau_a(t)$ is the set of timesteps in which arm $a$ was chosen. 

For any $t$,
\[
    \hat{\mu}_a(t) = \frac{\hat{\mu}_a^0(t)N_a(t) - \sum_{s \in \tau_a(t)}{\alpha_s}} {N_a(t)}. 
\]
Also, in round $t$ if the adversary attacked arm $a$, then
\[
    \hat\mu_a(t) = \hat\mu_K(t) - 2\beta(N_K(t)) - 3\sigma e^{N_a(t)}. 
\]
Consequently by the above two equations:
\begin{align*}
\displaybreak
    \frac{1}{N_a(t)}\sum_{s\in \tau_a(s)}\alpha_s &= \hat{\mu}_a^0(t) - \hat\mu_K(t) + 2\beta(N_K(t)) + 3\sigma e^{N_a(t)} \\
    &\le \Delta_a^+ + \beta(N_a(t)) + 3\beta(N_K(t)) + 3\sigma e^{N_a(t)} \\
    &\le \Delta_a^+ + \beta(N_a(t)) + 3\beta(N_K(t)) + 3\sigma e^{0.5\log{\log t} + 1}\\
    &\le \Delta_a^+ + 4\beta(N_a(t)) + 3e\cdot \sigma \sqrt{\log t}. 
 \end{align*}
Here, the third line follows from event $E$, and the last line comes from the fact that $\beta$ is nonincreasing and $N_a(t) < N_K(t)$. 
Thus focusing the attack cost spent on arm $a$:
\begin{align*}
    \sum_{s\in \tau_a(t)} \alpha_s &\le N_a(t)(\Delta_a^+ + 4\beta(N_a(t)) + 3e\cdot \sigma\sqrt{\log t}) \\
    &= \hatO(\sigma\sqrt{\log t} + \Delta_a^+). 
\end{align*}
Summing over all non-target arms, the total attack cost is $\hatO (K\sigma\sqrt{\log T} + \sum_{a\in [K]} \Delta_a^+)$. 
\endproof

%% file: appendix/proofs_TS.tex
\section{\MakeUppercase{Proofs: Attack Strategy against Thompson Sampling}}
This section details the proof of~\Cref{thm:TSAttack}. 
I first give a concentration result that will be useful in this section. Let $\zeta(t) := \sqrt{2 \log\frac{\pi^2 K t^2}{3\delta}}$. Denote the event in the following lemma by $E_1$. 
\begin{lemma}
\label{lem:eventE1}
With probability $1 - \delta$, for any round $t$, we have $\abs{\nu_a(t) - \hatmu_a(t)} < \zeta(t) / N_a(t)$. 
\end{lemma}
\begin{proof}
Fix round $t$ and an arm $a$. Then by a standard Gaussian tail bound:
\begin{align*}
\Pr[\abs {\nu_a(t) - \hatmu_a(t)} > \zeta(t) / N_a(t)] &< 2\exp(-\zeta(t)^2 / 2 ) \\
&= \frac{6\delta}{\pi^2}\cdot \frac{1}{K t^2}
\end{align*}
The lemma then follows from a union bound over $t$ and arms $a$. 
\end{proof}

\begin{lemma}
Assume event $E$ and $E_1$ hold. At any round $t$, arm $a$ is pulled for $\ceil{0.5\cdot \log\log t}$ times for any non-target arm $a\neq K$. 
\end{lemma}
\begin{proof}
Assume for the sake of contradiction arm $a$ is pulled more than $\ceil{0.5\log\log t}$ times. Suppose at round $t_0$ arm $a$ is pulled for the $\ceil{0.5\log\log t}$-th time. After round $t_0$, we must have:
\begin{align}
\hatmu_{a}(t_0) &\le \hatmu_K(t_0) - 2 \beta(N_K(t_0)) - 4 \exp(N_{a}(t)) - \sqrt{8 \log\frac{\pi^2 K}{3\delta}} \nonumber \\
&\le \hatmu_K(t_0) - 2 \beta(N_K(t_0)) - 4 \sqrt{\log t} - \sqrt{ 8 \log\frac{\pi^2 K}{3\delta}} \label{eq:TSGap}
\end{align}

Assume arm $a$ has been chosen again at round $t_1 > t_0$. Then we must have
\[
\nu_a(t_1) > \nu_K(t_1). 
\]
However, 
\begin{align*}
\nu_a(t_1) & < \hatmu_{a}(t_0) + \zeta(t_1)\\
& < \hatmu_{K}(t_0) - 2\beta(N_K(t_0)) - 4 \sqrt{\log t} - \sqrt{8 \log\frac{\pi^2 K}{3\delta}}+ \zeta(t_1) \\
& < \hatmu_K(t_1) - 4 \sqrt{\log t} - \sqrt{8 \log\frac{\pi^2 K}{3\delta}} + \zeta(t_1) \\
& < \hatmu_K(t_1) - \zeta(t_1) \\
& < \nu_K(t_1)
\end{align*}

Here, the first line follows from event $E_1$, the second line follows from the design of the attack strategy (\cref{eq:TSGap}), the third line follows from event $E$, the fourth line follows from the definition of $\zeta$, and the final line follows from event $E_1$ again. Hence, a contradiction is established, and arm $a$ will not be picked again before round $t$. 
\end{proof}

\begin{proof} (of~\Cref{thm:TSAttack})
In a similar fashion as the proof in~\cref{thm:UCBAttack}, the attack cost on arm $a$ can be bounded by:
\begin{align*}
\frac{1}{N_a(t)}\sum_{s\in\tau_a(s)} \alpha_s &\le \hatmu^0_a(t) - \hatmu_K(t) + 2\beta(N_K(t)) + 2.9\exp(N_a(t)) + 2\sqrt{\log\frac{\pi^2 K}{3\delta}} \\
&\le \Delta_a^+ + 4\beta(N_a(t)) + 2.9\sqrt{\log T} + 2\sqrt{\log\frac{\pi^2 K}{3\delta}}. 
\end{align*}

Hence the attack cost on arm $a$ can be bounded as: 
\begin{align*}
\sum_{s\in \tau_a(s)} \alpha_s = \hatO(\Delta_a^+ + \sqrt{\log T} + (\sigma + 1)\sqrt{\log\frac{\pi^2 K }{3\delta}}). 
\end{align*}
Summing over all non-target arms completes the proof. 
\end{proof}

%% file: appendix/proofs_lowerBound.tex
\section{\uppercase{Proofs: Lower Bounds}}
Recall that $\tau_a(t)$ represents the set of timesteps that arm $a$ was chosen before round $t$. Let $C_{a}(t) = \sum_{s\in \tau_a(t)} |\alpha_s|$ denote the cumulative attack cost spent on arm $a$ until round $t$. Note that 
\begin{align}
    \hat{\mu}_a(t) - \frac{C_{a}(t)}{N_a(t)} \le \hat{\mu}^0_a(t) \le \hat{\mu}_a(t) + \frac{C_{a}(t)}{N_a(t)}. 
\label{eq:attackDeviation}
\end{align}

\subsection{UCB Lower Bound}

\begin{proof} (of~\Cref{thm:UCBLowerBound})
Throughout this proof assume event $E$ holds with $\delta = 0.1$. Suppose the non-target arm has been pulled no more than $T / 26$ times. We show the attack cost is at least $\Delta + 0.22\sigma\sqrt{\log T - 1}$. 
Consider the last round the target arm is pulled. Denote this timestep by $t$, then $t > T/2$. Comparing the UCB index we must have
\[
    \hat{\mu}_2(t) + 3\sigma\sqrt{\frac{\log t}{N_2(t)}} > \hat{\mu}_1(t) + 3\sigma\sqrt{\frac{\log t}{N_1(t)}}. 
\]
Therefore by event $E$ and~\cref{eq:attackDeviation}
\begin{align*}
    \mu_2(t) + \beta(N_2(t)) + \frac{C_{2}(t)}{N_2(t)} + 3\sigma\sqrt{\frac{\log t}{N_2(t)}} \\
    > {\mu}_1(t) - \beta(N_1(t)) - \frac{C_{1}(t)}{N_1(t)} + 3\sigma\sqrt{\frac{\log t}{N_1(t)}}. 
\end{align*}
By the fact that $N_1(t) < N_2(t) / 25$:
\[
\sqrt{\frac{\log t}{N_2(t)}} < 0.2\sqrt{\frac{\log t}{N_1(t)}}, 
\]
and we can also verify
\[
\beta(N_2(t)) <  0.29\beta(N_1(t)). 
\]
Hence
\begin{align*}
    &\frac{C_1(t) + C_2(t)}{N_1(t)}\\
    &> \Delta - \beta(N_1(t)) - \beta(N_2(t)) + 3\sigma\sqrt{\frac{\log t}{N_1(t)}} - 3\sigma\sqrt{\frac{\log t}{N_2(t)}} \\
    &\ge \Delta - 1.29\beta(N_1(t)) + 2.8\sigma\sqrt{\frac{\log t}{N_1(t)}} \\
    &= \Delta - 1.29\sqrt{\frac{2\sigma^2}{N_1(t)}\log\frac{2\pi^2 N_1(t)^2}{3\delta}} + 2.8\sigma\sqrt{\frac{\log t}{N_1(t)}}\\
    &\ge \Delta + 0.22\sigma\sqrt{\frac{\log t}{N_1(t)}}. 
\end{align*}
Finally,
\begin{align*}
    C_1(t) + C_2(t) 
    &\ge N_1(t)\Delta + 0.22\sigma\sqrt{N_1(t)\log t} \\
    &\ge \Delta + 0.22\sigma\sqrt{\log t}. 
\end{align*}
This finishes the proof. 
\end{proof}



\subsection{Thompson Sampling Lower Bound}
In the following let $\Pr[\cN(\mu, \sigma^2) > t]$ denote the complementary CDF (tail probability) of a Gaussian random variable with specified mean and variance. 


\begin{lemma} [\cite{abramowitz1988handbook}]
The tail for a Gaussian distribution can be lower bounded by: 
\[
\Pr[\cN(-\mu, 1) > 0] > \sqrt{\frac{2}{\pi}}\frac{\exp(-\mu^2 / 2)}{\mu +\sqrt{\mu^2 + 2}}. 
\]
\label{lemma:gaussinTailLowerBound}
\end{lemma}

\begin{lemma}
\label{lem:PickNontarget-LowerBound}
Let $C$ be the total attack cost. For a round $t$, if $N_1(t) < N_2(t)$, then the probability that the non-target arm gets pulled is at least $\min\left( \Pr[ \cN(\frac{\Delta - C}{N_1(t)}, \frac{1}{N_1(t)}) > 0 ], \nicefrac{1}{2} \right)$. 
\end{lemma}
\begin{proof}
Fix a round $t$. If $\hatmu_1(t) > \hatmu_2(t)$, then $\Pr[a_t = 1] > 1/2$. Now assume $\hatmu_1(t) < \hatmu_2(t)$. 
\begin{align*}
    \Pr[a_t = 1] &= \Pr[\nu(t) > 0] \\
    &= \Pr[ \cN(\hatmu_1(t) - \hatmu_2(t), \frac{1}{N_1(t)} + \frac{1}{N_2(t)}) > 0] \\
    &\ge \Pr [ \cN(\hatmu_1^0(t) - \hatmu_2^0(t) - \frac{C}{N_1(t)}, \frac{1}{N_1(t)} + \frac{1}{N_2(t)}) > 0] \\
    &\ge \Pr[ \cN(\mu_1(t) - \mu_2(t) - \frac{C}{N_1(t)}, \frac{1}{N_1(t)} ) > 0] \\
    &= \Pr[ \cN( \Delta - \frac{C}{N_1(t)}, \frac{1}{N_1(t)} ) > 0 ]\\
    &\ge \Pr [\cN(\frac{\Delta - C}{N_1(t)}, \frac{1}{N_1(t)} ) > 0] 
\end{align*}
Here, the third line is because of \cref{eq:attackDeviation}, the fourth line is because the tail probability cannot increase if we decrease the variance of the Gaussian. 
\end{proof}

\begin{lemma}
Assume the adversary spends an attack cost no more than $C := \Delta + 0.1\sqrt{\log T}$. Then the non-target arm has been chosen $T^{0.8}$ times during period $[T/5, T/2]$. 
\end{lemma}

\begin{proof}
Consider rounds during the period $[T/5, T/2]$. If at any point $t$, $N_1(t) > N_2(t)$, then the non-target arm has already been chosen $T/10$ times. Hence we can assume $N_1(t) < N_2(t)$. 

\begin{align*}
\Pr[a_t = 1] &\ge \Pr[\cN( \frac{\Delta - C}{N_1(t)}, \frac{1}{N_1(t)}) > 0] \\
&\ge \Pr[ \cN({-0.1\sqrt{\log T}}, 1) > 0] \\
&\ge  \frac{1}{\sqrt{2\pi}} \cdot\frac{1}{0.2 \sqrt{\log T}} \cdot \exp( -{0.005\log T} ) \\
&> T^{-0.1}
\end{align*}



The expected number of pulls during this period is at least $0.3 T^{0.9}$. For a sufficiently large $T$, by a simple Hoeffding inequality, with probability 0.9, the learner has chosen the non-target arm for at least $T^{0.8}$ rounds during this period. 
\end{proof}

\begin{proof} (of~\Cref{thm:TSLowerBound})
Assume the adversary spends an attack cost no more than $C := \Delta + 0.1\sqrt{\log T}$. We will show with probability at least $0.8$, the learner chose the non-target arm more than $T/10$ times. By the previous lemma, the learner choose the non-target arm at least $T^{0.8}$ times during the period $[T/5, T/2]$. 

Now, consider rounds during the period $[T/2, T]$. The following holds: \szdelete{The quantity $\nu_1(t) - \nu_2(t)$ is a random variable with mean greater than $\frac{-1}{T^{0.5}}$. Hence applying the Gaussian tail lower bound: }

\begin{align*}
\Pr[a_t = 1] &= \Pr[ \cN( \frac{\Delta - C}{N_1(t)} , \frac{1}{N_1(t)} ) > 0 ] \\
&\ge \Pr[ \cN( \frac{\Delta - C}{T^{0.8}} , \frac{1}{T} ) > 0 ] \\
&\ge \Pr[ \cN (\frac{-0.1\sqrt{\log T}}{T^{0.8}}, \frac{1}{T}) > 0] \\
&\ge \Pr[ \cN(\frac{-0.1}{T^{0.5}}, \frac{1}{T} ) > 0] \\
&\ge \Pr[ \cN({-0.1}, 1 ) > 0] \\
&> 0.4. 
\end{align*}

The expected number of pulls during this period is at least $T/5$. For a sufficiently large $T$, by Hoeffding inequality, with probability $0.9$, the learner chooses the non-target arm at least $T / 10$ times. 

Hence, with a probability of at least 0.8, the learner has chosen the non-target arm at least $T/10$ times. 
\end{proof}

\subsection{$\eps$-greedy Lower Bound}
I first prove a lemma that gives a tight characterization of the number of times each arm is pulled in exploration rounds. 
\begin{lemma}
Fix $\delta \in (0,1)$. Suppose $T$ satisfies $\sum_{t=1}^T \nicefrac{c}{t} \ge 16\log(\nicefrac{4}{\delta})$, then with probability $1 - \delta$, the number of times each arm is pulled during exploration rounds is between $0.5c\log T$ and $2c\log T$. \label{lemma:epsgreedy-explorerounds}
\end{lemma}

\begin{proof}
Fix arm $a$. Let $X_t$ be the indicator variable that takes the value 1 if arm $a$ was pulled in round $t$ as exploration. Then
\begin{align*}
    \Ex[X_t] &= \frac{c}{t}\\
    \Var[X_t] &= \frac{c}{t} (1 - \frac{c}{t}). 
\end{align*}
Then by a Freedmans' style inequality (e.g.~\cite{agarwal2014taming}), for any $\eta \in(0,1)$, with probability $1 - \nicefrac{\delta}{4}$, we have
\begin{align*}
\sum_{t=1}^T (X_t - \frac{c}{t}) &\le \eta \sum_{t=1}^T \Var[X_t] + \frac{\log(4/\delta)}{\eta}\\
&\le \eta \sum_{t=1}^T \Ex[X_t] + \frac{\log(4/\delta)}{\eta}\\
&= \eta \sum_{t=1}^T \frac{c}{t} + \frac{\log (4/\delta)}{\eta}. 
\end{align*}
Choosing $\eta = \sqrt{ \frac{\log(4/\delta)}{\sum_{t=1}^T c/t} }$, we obtain
\begin{align*}
    \sum_{t=1}^T X_t < \sum_{t=1}^T\frac{c}{t} + 2\sqrt{\sum_{t=1}^T \frac{c}{t}\log(4/\delta)}. 
\end{align*}

A lower bound on $\sum_{t=1}^T{X_t}$ is similar by taking the random variables to be $-X_t$ instead of $X_t$ in Freedmans' inequality, and we can show with probability $1 - \nicefrac{\delta}{4}$
\begin{align*}
    \sum_{t=1}^T X_t > \sum_{t=1}^T \frac{c}{t} - 2\sqrt{\sum_{t=1}^T \frac{c}{t} \log(4/\delta)}. 
\end{align*}
Thus with probability $1 - \nicefrac{\delta}{2}$, for $T$ large enough
\[
\frac{1}{2}\sum_{t=1}^T \frac{c}{t} < \sum_{t=1}^T X_t < 2\sum_{t=1}^T \frac{c}{t}. 
\]
The lemma then follows by taking a union bound over the 2 arms. 
\end{proof}

\begin{proof} (of~\Cref{thm:eps-greedyLowerBound})
Throughout this proof assume event $E$ and~\cref{lemma:epsgreedy-explorerounds} holds with $\delta = 0.1$. By a union bound, these events hold simultaneously with probability at least $0.8$. Assume the target arm has been pulled at least $3T/4$ rounds. We will show the adversary spend an attack cost at least $\Omega(\Delta \log T)$. 
Consider the last exploitation round before $T$ in which the learner pulled the 2nd arm, and denote the timestep by $t$. By round $T$, the number of times the 2nd arm was pulled in exploration rounds is at most $2c\log T$. Thus to ensure the 2nd arm is pulled no less than $T - T / 4$ rounds, we must have
\[
    t > T - T/4 - 2c\log T > T / 2. 
\]
In this round, the post-attack mean of the 2nd arm must be higher than that of the 1st arm:
\[
    \hat{\mu}_2(t) > \hat{\mu}_1(t). 
\]
Therefore by event $E$ and~\cref{eq:attackDeviation}:
\begin{align*}
    \mu_2(t) + \beta(N_2(t)) + \frac{C_{2}(t)}{N_2(t)} &> \mu_1(t) - \beta(N_1(t)) - \frac{C_1(t)}{N_1(t)}
\end{align*}
leading to
\begin{align*}
    C_1(t) + C_2(t) &> N_1(t)(\Delta - 2\beta(N_1(t))) \\
    &> c\cdot \Delta \log T/ 6,
\end{align*}
since assuming~\cref{lemma:epsgreedy-explorerounds} holds, $N_1(t) > 0.5c\log T$, and by the assumption on $T$ we have $\Delta > 3\beta(0.5c\log T) > 3\beta(N_1(t))$. 
This finishes the proof. 
\end{proof}

%% file: appendix/proofs_smooth.tex
\section{\MakeUppercase{Proofs: Competitive Ratio with Smoothed Responses}}
\begin{lemma}
\label{lemma:smoothedEnoughPulls}
Assume for each arm $a$, the probability of pulling $a$ in each round $t$ can be lower bounded by $\rho$. Then with probability $1 - \delta$, for every $t > \frac{10}{\rho} \cdot\log\frac{K}{\delta}$, the following holds:
\begin{align*}
N_a(t) \ge \frac{\rho t}{2}. 
\end{align*}
Moreover, with probability $1 - 2\delta$, for any $t > t_0 := \max(\frac{10}{\rho}\cdot \log\frac{K}{\delta}, \frac{16 C}{\rho \Delta_a}, \frac{20\log (T/\delta)}{\rho \Delta_a^2})$, the learner has:
\[
\hatmu_a(t) < \hatmu_{a^*} - \Delta_a/2. 
\]
\end{lemma}
\begin{proof}
Fix arm $a$ and round $t$. With probability $1 - \delta$, by Chernoff bound:
\begin{align*}
\Pr[N_a(t) < \frac{\rho t}{ 2 }] &< \exp(-0.25\rho t / 2) 
\end{align*}
Moreover, it is easy to verify that for any $t > \frac{10}{\rho} \cdot\log\frac{K}{\delta}$:
\begin{align*}
\exp(-0.25\rho t / 2) &\le \frac{\delta}{K t^2}. 
\end{align*}
A union bound then completes the first part of the proof. 

Now, fix some suboptimal arm $a$ and let $t > t_0$. Then $N_a(t) > \max(\frac{8C}{\Delta_a}, \frac{10\log T}{\Delta_a^2})$. Consequently by \cref{eq:attackDeviation},
\begin{align*}
\hatmu_a(t) &\le \hatmu^0_a(t) + \frac{C}{N_a(t)} \le \hatmu^0_a(t) + \frac{\Delta_a}{8}, \\
\hatmu_{a^*}(t) &\ge \hatmu^0_{a^*}(t) - \frac{C}{N_a(t)} \ge \hatmu^0_{a^*}(t) - \frac{\Delta_a}{8}. 
\end{align*}
And by event $E$, 
\begin{align*}
\hatmu^0_a(t) &\le \hatmu_a + \sqrt{ \frac{\log (T/\delta)}{N_a(t)} } \le \mu_a + \frac{\Delta_a}{8}, \\
\hatmu^0_{a^*}(t) &\le \hatmu_{a^*} + \sqrt{ \frac{\log (T/\delta)}{ N_a(t)} } \ge \mu_{a^*} - \frac{\Delta_a}{8}. 
\end{align*}
Rearranging the above inequalities completes the proof. 
\end{proof}

\subsection{Proof for Smoothed Myopic Response}


\begin{proof} (of \Cref{thm:smoothMyopic})
Choose $\delta = 1/T$ in \Cref{lemma:smoothedEnoughPulls}. After round $t_0$, any suboptimal arm has a probability $\rho$ of being pulled. Hence the total pulls of some suboptimal arm $a$ can be bounded by: 
\[
O ( \frac{C}{\Delta_a} + \frac{\log T}{\Delta_a^2}) + 2 \rho T. 
\]
The total regret contributed by suboptimal arm $a$ can be bounded by
\[
O({C}+ \frac{\log T}{\Delta_a}) + 2 \rho T \Delta_a \le o(T) + 2\rho T \mu^*. 
\]

The optimal reward is $T\mu^* $. The expected collected reward is:
\begin{align*}
\REW &\ge T\mu^* - \sum_{a} 2 \rho T \mu^* - o(T) \\
&\ge T\mu^* - 2 \rho K  T \mu^* - o(T) \\
&\ge (1 - 2\rho K) T\mu^* - o(T). 
\end{align*}
Therefore setting $\rho = \eps / 2K$ achieves the final regret. 
\end{proof}

\subsection{Proof for Quantal Response}
In the following let $\psi_a(t) = \frac{\hatmu_a(t)}{\hatmu^*(t)}$ as in the algorithm, then $\psi_a(t)\in [0,1]$. Also let $\rho(\lambda) = \frac{1}{K\exp(\lambda)}$. 
\begin{lemma}
For any round $t$ and any arm $a$, $p_{t,a} > \rho(\lambda)$. 
\end{lemma}

\begin{proof}
\begin{align*}
p_{t,a} &= \frac{\exp(\lambda \psi_a(t))}{\sum_b \exp(\lambda \psi_b(t))} \\
&= \frac{\exp(\lambda \psi_a(t))}{\exp(\lambda \psi_a(t)) + \sum_{b\neq a}{\exp(\lambda \psi_b(t)) }} \\
&\ge \frac{1}{ 1 + (K-1)\exp(\lambda)} \\
&\ge \frac{1}{K\exp(\lambda)}. \qedhere
\end{align*}
\end{proof}




\proof (of \Cref{thm:quantal})

Choose $\delta = 1/T$ in \Cref{lemma:smoothedEnoughPulls}. After $t > t_0$, $\hatmu_a(t) < \hatmu_{a^*}(t) - \Delta_a / 2$. Thus 
\[
\psi_a(t) \le \frac{\hatmu_{a^*}(t) - \Delta_a/2}{\hatmu_{a^*}(t)} \le 1 - \frac{\Delta_a / 2}{\mu^* + \Delta_a / 2} \le 1 - \frac{1}{4} \frac{\Delta_a}{\mu^*}. 
\] The probability that arm $a$ gets pulled after $t_0$ is then:
\begin{align*}
p_{t,a} &\le \frac{\exp(\lambda \psi_a(t))}{\exp(\lambda \psi_a(t)) + \exp(\lambda \psi_{a^*}(t)))} \\
&\le \frac{1}{ 1 + \exp(\lambda - \lambda \psi_a(t)) }\\
&\le \frac{1} {2 \exp(\lambda - \lambda \psi_a(t))} \\
&\le \frac{1}{2\exp(\frac{\lambda \Delta_a}{2\mu^*})} \\
&\le \frac{\mu^*}{\lambda \Delta_a}. 
\end{align*}
The total regret contributed by arm $a$ can be bounded by:
\begin{align*}
O(\frac{C}{\rho(\lambda)} + \frac{\log T}{\rho(\lambda) \Delta_a}) + \frac{2\mu^*}{\lambda \Delta_a} \cdot T \cdot\Delta_a &\le o(T) + \frac{2 T \mu^* }{\lambda}. 
\end{align*}
Hence the total collected reward can be bounded by:
\begin{align*}
\REW &\ge T\mu^* -  \frac{2K}{\lambda} T\mu^* - o(T) \\
&\ge (1 - \frac{2K}{\lambda}) T\mu^* - o(T). 
\end{align*}
Choosing $\lambda = \frac{2K}{\varepsilon}$ finishes proof. 
\endproof